\documentclass[nohyperref]{article}

\usepackage{microtype}
\usepackage{graphicx}
\usepackage{subfigure}
\usepackage{booktabs} 

\usepackage{hyperref}

\usepackage[accepted]{icml2022}

\usepackage{amsmath}
\usepackage{amssymb}
\usepackage{mathtools}
\usepackage{amsthm}

\usepackage{color}
\usepackage{soul}

\usepackage[capitalize,noabbrev]{cleveref}

\newcommand\Tstrut{\rule{0pt}{2.2ex}}         

\theoremstyle{plain}
\newtheorem{theorem}{Theorem}[section]

\theoremstyle{definition}

\newtheorem{assumption}[theorem]{Assumption}
\theoremstyle{remark}
\newtheorem{remark}[theorem]{Remark}

\usepackage[textsize=tiny]{todonotes}

\icmltitlerunning{Fisher SAM: Information Geometry and Sharpness Aware Minimisation}

\begin{document}

\twocolumn[
\icmltitle{Fisher SAM: Information Geometry and Sharpness Aware Minimisation}



\icmlsetsymbol{equal}{*}

\begin{icmlauthorlist}
\icmlauthor{Minyoung Kim}{saicc}
\icmlauthor{Da Li}{saicc}
\icmlauthor{Shell Xu Hu}{saicc}
\icmlauthor{Timothy M.~Hospedales}{saicc,ed}
\end{icmlauthorlist}

\icmlaffiliation{saicc}{Samsung AI Center, Cambridge, UK}
\icmlaffiliation{ed}{University of Edinburgh}

\icmlcorrespondingauthor{Minyoung Kim}{mikim21@gmail.com}

\icmlkeywords{Machine Learning, ICML}

\vskip 0.3in
]



\printAffiliationsAndNotice{}  

\begin{abstract}
Recent sharpness-aware minimisation (SAM) is known to find flat minima which is beneficial for better generalisation with improved robustness. SAM essentially modifies the loss function by reporting the maximum loss value within the small neighborhood around the current iterate.  However, it uses the Euclidean ball to define the neighborhood, which can be inaccurate since loss functions for neural networks are typically defined over probability distributions (e.g., class predictive probabilities), rendering the parameter space non Euclidean. In this paper we 
consider the information geometry of the model parameter space when defining the neighborhood, namely replacing SAM's Euclidean balls with ellipsoids induced by the Fisher information. Our approach, dubbed Fisher SAM, 
defines more accurate neighborhood structures 
that conform to the intrinsic metric of the underlying statistical manifold. For instance, SAM may probe the worst-case loss value at either a too nearby or inappropriately distant point due to the ignorance of the parameter space geometry, which is avoided by our Fisher SAM. 
Another recent Adaptive SAM approach  stretches/shrinks the Euclidean ball in accordance with the scale of the parameter magnitudes. This might be dangerous, potentially destroying the neighborhood structure. We demonstrate improved performance of the proposed Fisher SAM on several benchmark datasets/tasks.
\end{abstract}

\section{Introduction}\label{sec:intro}
Contemporary deep learning models achieve state of the art generalisation performance on a wide variety of tasks. These models are often massively  overparameterised, and capable of memorizing the entire training set \cite{zhang2017understandingDeepGen}. The training loss landscape of such models is complex and non-convex with multiple local and global minima of highly varying generalisation performance. Good performance is therefore obtained by exploiting various explicit and implicit regularisation schemes during learning to find local minima in the training loss that actually generalise well. Methods such as dropout \cite{srivastava2014dropout}, weight-decay, and data augmentation have been developed to provide explicit regularisation, while the dynamics of optimisers such as SGD can provide implicit regularisation, by finding solutions with low-norm weights \cite{chaudhar2017entropySGD,zhang2017understandingDeepGen}. 


A number of studies have linked the flatness of a given training minima to generalisation quality \cite{keskar2017batchGen,chaudhar2017entropySGD}. Searching for flat minima of the loss function is intuitively appealing, as it is obviously beneficial for finding models resilient to data noise and/or model parameter corruption/perturbation. This has led to an increasing number of optimisation methods ~\cite{chaudhar2017entropySGD,sam,adv_def} designed to explicitly search for flat minima. Despite this variety of noteworthy theoretical and empirical work, existing approaches have yet to scalably solve this problem, as developing computationally efficient methods for finding flat minima is non-trivial. 

A seminal method in this area is known as sharpness-aware minimisation (SAM) \cite{sam}. SAM is a mini-max type algorithm that essentially modifies the loss function to report the maximum loss value within the small neighborhood around the current iterate. Optimising with SAM thus prefers flatter minima than conventional SGD. 
However, one of the main drawbacks of SAM is that it uses a Euclidean ball to define the neighborhood, which is inaccurate since loss functions for neural networks are typically defined over probability distributions (e.g., class predictive probabilities), rendering the parameter space non Euclidean. 
Another recent approach called Adaptive SAM (ASAM) \cite{asam} stretches/shrinks the Euclidean ball in accordance with the scales of the parameter magnitudes. However, this approach to determining the flatness ellipsoid of interest is heuristic and might severely degrade the neighborhood structure. Although SAM and ASAM are successful in many empirical tasks, ignorance of the underlying geometry of the model parameter space may lead to suboptimal results. 

In this paper we build upon the ideas of SAM, but address the issue of a principled approach to determining the ellipsoid of interest by considering information geometry~\cite{amari98,murray_rice} of the model parameter space when defining the neighborhood. Specifically, we replace SAM's Euclidean balls with ellipsoids induced by the Fisher information. Our approach, dubbed Fisher SAM, defines more accurate neighborhood structures that conform to the intrinsic metric of the underlying statistical manifold. By way of comparison, SAM may probe the worst-case loss value at either a too nearby or too far point due to using a spherical neighborhood. In contrast Fisher SAM avoids this by probing the worst-case point within the ellipsoid derived from the Fisher information at the current point -- thus providing a more principled and optimisation objective, and improving empirical generalisation performance. 

Our main contributions are as follows:
\begin{enumerate}
\itemsep0em
\item We propose a novel information geometry and sharpness aware loss function which addresses the abovementioned issues of the existing flat-minima optimisation approaches.
\item Our Fisher SAM is as efficient as SAM, only requiring double the cost of that of  vanilla SGD, using the gradient magnitude approximation for Fisher information matrix. We also justify this approximation. 
\item We provide a theoretical generalisation bound similar to SAM's using the prior covering proof technique in PAC-Bayes, in which we extend SAM's spherical Gaussian prior set to an ellipsoidal full-covariance set. 
%
\item We demonstrate  improved empirical performance of the proposed FSAM on several benchmark datasets and tasks: image classification, ImageNet overtraining, finetuning; and label-noise robust learning; and robustness to parameter perturbation during inference. 
\end{enumerate}


\section{Background}\label{sec:background}

 
Although flatness/sharpness of the loss function 
can be formally defined using the Hessian, dealing with (optimizing) the Hessian function is computationally prohibitive. As a remedy, the sharpness-aware minimisation (SAM for short)~\cite{sam} introduced a novel robust loss function, where the new loss at the current iterate is defined as the maximum (worst-case) possible loss within the neighborhood at around it. More formally, considering a $\gamma$-ball neighborhood, the robust loss $\l^\gamma$ is defined as:
\begin{align}
l^\gamma(\theta) = \max_{\|\epsilon\|\leq \gamma} l(\theta+\epsilon),
\label{eq:sam_loss}
\end{align}
where $\theta$ is the model parameters (iterate), and $l(\theta)$ is the original loss function. Using the first-order Taylor (linear) approximation of $l(\theta+\epsilon)$, 
(\ref{eq:sam_loss}) becomes the famous dual-norm problem~\cite{cvx}, admitting a closed-form solution. In the Euclidean (L2) norm case, the solution becomes the normalised gradient,
\begin{align}
\epsilon^*_{SAM}(\theta) = \gamma \frac{\nabla l(\theta)}{\|\nabla l(\theta)\|}.
\label{eq:sam_eps}
\end{align}
Plugging (\ref{eq:sam_eps}) into (\ref{eq:sam_loss}) defines the SAM loss, while its gradient can be further simplified by ignoring the (higher-order) gradient terms in $\nabla \epsilon^*(\theta)$ for computational tractability:
\begin{align}
&l^\gamma_{SAM}(\theta) = l(\theta'), \ \ 
\nabla l^\gamma_{SAM}(\theta) = \frac{\partial l(\theta)}{\partial \theta} \bigg|_{\theta=\theta'} \label{eq:sam_loss_grad} \\ 
&\ \ \ \ \ \ \ \ \textrm{where} \ \ \theta' = \theta + \epsilon^*_{SAM}(\theta). \nonumber
\end{align}
In terms of computational complexity, SAM incurs only twice the forward/backward cost of the standard SGD: one forward/backward for computing $\epsilon^*_{SAM}(\theta)$ and the other for evaluating the loss and gradient at $\theta' = \theta + \epsilon^*_{SAM}(\theta)$.

More recently, a drawback of SAM, related to the model parameterisation, was raised by~\cite{asam}, in which SAM's fixed-radius $\gamma$-ball can be sensitive to the parameter re-scaling, weakening the connection between sharpness and generalisation performance. To address the issue, they proposed what is called Adaptive SAM (ASAM for short), which essentially re-defines the neighborhood $\gamma$-ball with the magnitude-scaled parameters. That is,  
\begin{align}
l^\gamma_{ASAM}(\theta) = \max_{\|\epsilon / |\theta| \|\leq \gamma} l(\theta+\epsilon),
\label{eq:asam_loss}
\end{align}
where $\epsilon / |\theta|$ is the elementwise operation (i.e., $\epsilon_i/|\theta_i|$ for each axis $i$). It leads to the following maximum (worst-case) probe direction within the neighborhood,
\begin{align}
\epsilon^*_{ASAM}(\theta) = \gamma \frac{\theta^2 \nabla l(\theta)}{\|\theta \nabla l(\theta) \|} \ \ \ \ \textrm{(elementwise ops.)}.
\label{eq:asam_eps}
\end{align}
The loss and gradient of ASAM are defined similarly as (\ref{eq:sam_loss_grad}) with $\theta' = \theta + \epsilon^*_{ASAM}(\theta)$.

\section{Our Method: Fisher SAM}\label{sec:main}

ASAM's $\gamma$-neighborhood structure is a function of $\theta$, thus not fixed but adaptive to parameter scales in a quite intuitive way (e.g., more perturbation allowed for larger $\theta_i$, and vice versa). However, ASAM's parameter magnitude-scaled neighborhood choice is rather ad hoc, not fully reflecting the underlying geometry of the parameter manifold. 

Note that the loss functions for neural networks are  typically dependent on the model parameters $\theta$ only through the predictive distributions $p(y|x,\theta)$ where $y$ is the target variable (e.g., the negative log-likelihood or cross-entropy loss,  $l(\theta) = \mathbb{E}_{x,y}[ -\log p(y|x,\theta)]$). Hence the geometry of the parameter space is not Euclidean but a {\em statistical manifold} induced by the Fisher information metric of the distribution $p(y|x,\theta)$~\cite{amari98,murray_rice}.

The intuition behind the Fisher information and statistical manifold can be informally stated as follows. When we measure the distance between two neural networks with parameters $\theta$ and $\theta'$, we should compare the underlying distributions $p(y|x,\theta)$ and $p(y|x,\theta')$. The Euclidean distance of the parameters $\|\theta-\theta'\|$ does not capture this distributional divergence because two distributions may be similar even though $\theta$ and $\theta'$ are largely different (in L2 sense), and vice versa. For instance, even though $p(x|\theta) = \mathcal{N}(\mu, 1+0.001\sigma)$ and $p(x|\theta') =  \mathcal{N}(\mu', 1+0.001\sigma')$ with $\theta=(\mu=1, \sigma=10)$ and $\theta'=(\mu'=1, \sigma'=20)$ have large L2 distance, the underlying distributions are nearly the same. That is, the Euclidean distance is not a good metric for the parameters of a distribution family. We need to use {\em statistical divergence} instead, such as the Kullback-Leibler (KL) divergence, from which the Fisher information metric can be derived. 

Based on the idea, we propose a new SAM algorithm that fully reflects the underlying geometry of the statistical manifold of the parameters. In (\ref{eq:sam_loss}) we replace the Euclidean distance by the KL divergence\footnote{To be more rigorous, one can consider the {\em symmetric} Jensen-Shannon divergence, $d(\theta',\theta) = 0.5 \cdot \mathbb{E}_x[\textrm{KL}(\theta'||\theta) + \textrm{KL}(\theta||\theta')]$. For $\theta'\approx \theta$, however, the latter KL term vanishes (easily verified using the derivations similar to those in Appendix~\ref{appsec:kl_fisher}), and it coincides with the KL divergence in (\ref{eq:fsam_loss0}) (up to a constant factor).
}:
\begin{align}
&l^\gamma_{FSAM}(\theta) = \max_{d(\theta+\epsilon,\theta) \leq \gamma^2} l(\theta+\epsilon) \ \ \textrm{where} 
\label{eq:fsam_loss0} \\
&\ \ \ \ \ \ \ \ 
d(\theta',\theta) = \mathbb{E}_x[\textrm{KL}(p(y|x,\theta')||p(y|x,\theta))], \nonumber
\end{align}
which we dub Fisher SAM (FSAM for short). 
For small $\epsilon$, it can be shown that $d(\theta+\epsilon,\theta) \approx \epsilon^\top F(\theta) \epsilon$ (See Appendix~\ref{appsec:kl_fisher} for details), where $F(\theta)$ is the Fisher information matrix,
\begin{align}
F(\theta) = \mathbb{E}_x\mathbb{E}_\theta 
\Big[
\nabla \log p(y|x,\theta) \nabla \log p(y|x,\theta)^\top \Big].
\label{eq:fisher_info}
\end{align}
That is, our Fisher SAM loss function can be written as:
\begin{align}
l^\gamma_{FSAM}(\theta) = \max_{\epsilon^\top F(\theta) \epsilon \leq \gamma^2} l(\theta+\epsilon).
\label{eq:fsam_loss}
\end{align}
We solve (\ref{eq:fsam_loss}) using the first-order approximated objective $l(\theta+\epsilon) \approx l(\theta) + \nabla l(\theta)^\top \epsilon$, leading to a quadratic constrained linear programming problem. 
The Lagrangian is
\begin{align}
\mathcal{L}(\epsilon,\lambda) = l(\theta) + \nabla l(\theta)^\top \epsilon - \lambda (\epsilon^\top F(\theta) \epsilon - \gamma^2),
\label{eq:lagrangian}
\end{align}
and solving $\frac{\partial \mathcal{L}}{\partial \epsilon}=0$ yields $\epsilon^* = \frac{1}{2\lambda} F(\theta)^{-1} \nabla l(\theta)$. Plugging this into the ellipsoidal constraint (from the KKT conditions) determines the optimal $\lambda$, resulting in:
\begin{align}
\epsilon^*_{FSAM}(\theta) = \gamma \frac{F(\theta)^{-1} \nabla l(\theta)}{\sqrt{\nabla l(\theta) F(\theta)^{-1} \nabla l(\theta)}}.
\label{eq:fsam_eps}
\end{align}
The loss and gradient of Fisher SAM are defined similarly as (\ref{eq:sam_loss_grad}) with $\theta' = \theta + \epsilon^*_{FSAM}(\theta)$.


\textbf{Approximating Fisher.} Dealing with a large dense matrix $F(\theta)$ (and its inverse) is prohibitively expensive. Following the conventional practice, we consider the empirical diagonalised minibatch approximation, 
\begin{align}
F(\theta) \approx \frac{1}{|B|} \sum_{i \in B} \textrm{Diag} \big( \nabla \log p(y_i|x_i,\theta) \big)^2,
\label{eq:fisher_approx}
\end{align}
for a minibatch $B = \{(x_i,y_i)\}$. $\textrm{Diag}(v)$ is a diagonal matrix with vector $v$ embedded in the diagonal entries. 
However, it is still computationally cumbersome to handle {\em instance-wise} gradients in (\ref{eq:fisher_approx}) using the off-the-shelf auto-differentiation numerical libraries such as PyTorch~\cite{pytorch}, Tensorflow~\cite{tf} or JAX~\cite{jax} that are especially tailored for the {\em batch sum} of gradients for the best efficiency. 
The sum of squared gradients in (\ref{eq:fisher_approx}) has a similar form as the Generalised Gauss-Newton (GGN) approximation for a Hessian~\cite{Schraudolph02,Graves11,Martens14}. Motivated from the {\em gradient magnitude} approximation of Hessian/GGN~\cite{gm,vprop}, we replace the sum of gradient squares with the square of the batch gradient sum,
\begin{align}
\hat{F}(\theta) = \textrm{Diag} \bigg( \frac{1}{|B|} \sum_{i \in B} \nabla \log p(y_i|x_i,\theta) \bigg)^2.
\label{eq:fisher_approx_gm}
\end{align}
Note that (\ref{eq:fisher_approx_gm}) only requires the gradient of the batch sum of the logits (prediction scores), a very common form efficiently done by the off-the-shelf auto-differentiation libraries. If we adopt the negative log-loss (cross-entropy), it further reduces to $\hat{F}(\theta) = \textrm{Diag}(\nabla l_B(\theta))^2$ where $l_B(\theta)$ is the minibatch estimate of $l(\theta)$. For the inverse of the Fisher information in (\ref{eq:fsam_eps}), we add a small positive regulariser to the diagonal elements before taking the reciprocal. 

Although this gradient magnitude approximation can introduce unwanted bias to the original $F(\theta)$ (the amount of bias being dependent on the degree of cross correlation between $\nabla \log p(y_i|x_i,\theta)$ terms), it is a widely adopted technique for learning rate scheduling also known as average squared gradients in modern optimisers such as RMSprop, Adam, and AdaGrad. Furthermore, the following theorem from~\cite{vprop} justifies the gradient magnitude approximation by relating the squared sum of vectors and the sum of squared vectors. 
\begin{theorem}[Rephrased from Theorem 1~\cite{vprop}]
Let $\{v_1,\dots,v_N\}$ be the population vectors, and $B \subset \{1\dots N\}$ be a uniformly sampled (w/ replacement) minibatch with $M=|B|$. Denoting the minibatch and population averages by $\overline{v}(B) = \frac{1}{M}\sum_{i\in B} v_i$ and $\overline{v} = \frac{1}{N}\sum_{i=1}^N v_i$, respectively, 
for some constant $\alpha$,
\begin{align}
\vspace{-1.0em}
\frac{1}{N}\sum_{i=1}^N v_i v_i^\top = \alpha \mathbb{E}_B[\overline{v}(B)\overline{v}(B)^\top] + (1-\alpha) \overline{v}\overline{v}^\top.
\label{eq:thm1}
\end{align}
\label{thm:1}
\vspace{-1.5em}
\end{theorem}
Although it is proved in~\cite{vprop}, we provide full proof here for self-containment. 
\begin{proof}
\vspace{-0.8em}
We denote by $\mathbb{V}_i(v_i)$ and  $\mathbb{V}_B(\cdot)$ the population variance and variance over $B$, respectively. 
Let $A$ be the LHS of (\ref{eq:thm1}). Then $\mathbb{V}_i(v_i) = A - \overline{v}\overline{v}^\top$. Also $\mathbb{V}_B(\overline{v}(B)) = \mathbb{E}_B[\overline{v}(B)\overline{v}(B)^\top] - \overline{v}\overline{v}^\top$ since $\mathbb{E}_B[\overline{v}(B)] = \overline{v}$. From Theorem 2.2 of~\cite{cochran}, $\mathbb{V}_B(\overline{v}(B)) = \frac{N-M}{M (N-1)} \mathbb{V}_i(v_i)$. By arranging the terms, with $\alpha=\frac{M (N-1)}{N-M}$, we have $A = \alpha \mathbb{E}_B[\overline{v}(B)\overline{v}(B)^\top] + (1-\alpha) \overline{v}\overline{v}^\top$.
\vspace{-0.8em}
\end{proof}
The theorem essentially implies that the sum of squared gradients (LHS of (\ref{eq:thm1})) gets close to the squared sum of gradients ($\overline{v}(B)\overline{v}(B)^\top$ or $\overline{v}\overline{v}^\top$) if the batch estimate $\overline{v}(B)$ is close enough to its population version $\overline{v}$ \footnote{For instance, the two terms in the RHS of (\ref{eq:thm1}) can be approximately merged into a single squared sum of gradients.}.

\begin{algorithm}[t!]
  \caption{Fisher SAM.}
  \label{alg:fsam}
\begin{small}
\begin{algorithmic}
   \STATE {\bfseries Input:} Training set $S=\{(x_i,y_i)\}$, neighborhood size $\gamma$, \\
   \ \ \ \ \ \ \ \ \ \ \ \ regulariser $\eta$ for inverse Fisher, and learning rate $\alpha$.
   \FOR{$t=1,2,\dots$}
      \STATE 1) Sample a batch $B \sim S$.
      \STATE 2) Compute the gradient of the batch loss $\nabla l_B(\theta)$.
      \STATE 3) Compute the approximate Fisher info $\hat{F}(\theta)$ as per (\ref{eq:fisher_approx_gm}). 
      \STATE 4) Compute $\epsilon_{FSAM}^*(\theta)$ using (\ref{eq:fsam_eps}). 
      \STATE 5) Compute gradient approximation for the Fisher SAM loss,  \\ 
      \ \ \ \ \ \ $\nabla l^\gamma_{FSAM}(\theta) = \frac{\partial l_B(\theta)}{\partial \theta} \Big|_{\theta+\epsilon_{FSAM}^*(\theta)}$.
      \STATE 6) Update: $\theta \leftarrow \theta - \alpha \nabla l^\gamma_{FSAM}(\theta)$.
   \ENDFOR
\end{algorithmic}
\end{small}
\end{algorithm}

The Fisher SAM algorithm\footnote{In the current version, we take the vanilla gradient update (step 6 in Alg.~\ref{alg:fsam}). However, it is possible to take the natural gradient update instead (by pre-multiplying the update vector by the inverse Fisher information), which can be beneficial for other methods SGD and SAM, likewise. Nevertheless, we leave it and related further extensive study as future work. 
} is summarized in Alg.~\ref{alg:fsam}. 
Now we state our main theorem for generalisation bound of Fisher SAM. Specifically we bound the expectation of the generalisation loss over the Gaussian perturbation that aligns with the Fisher information geometry. 
\begin{theorem}[Generalisation bound of Fisher SAM] 
Let $\Theta\subseteq\mathbb{R}^k$ be the model parameter space that satisfies the regularity conditions in Appendix~\ref{appsec:main_proof}. 
For any $\theta\in\Theta$, with probability at least $1-\delta$ over the choice of the training set $S$ ($|S|=n$), the following holds.
\begin{align}
\mathbb{E}_\epsilon
[l_D(\theta+\epsilon)] \leq  l^\gamma_{FSAM}(\theta; S) + \sqrt{\frac{O(k + \log \frac{n}{\delta})}{n-1}},
\label{eq:fsam_bound}
\end{align}
where $l_D(\cdot)$ is the generalisation loss, $l^\gamma_{FSAM}(\cdot; S)$ is the empirical Fisher SAM loss as in (\ref{eq:fsam_loss}), and the expectation is over $\epsilon\sim\mathcal{N}(0,\rho^2 F(\theta)^{-1})$ for $\rho\propto\gamma$. 
\label{thm:main}
\end{theorem}
\begin{remark}
Compared to SAM's generalisation bound in Appendix A.1 of~\cite{sam}, the complexity term is asymptotically identical (only some constants are different). However, the expected generalisation loss in the LHS of (\ref{eq:fsam_bound}) is different: we have perturbation of $\theta$ aligned with the Fisher geometry of the model parameter space (i.e., $\epsilon\sim\mathcal{N}(0,\rho^2 F(\theta)^{-1})$), while in SAM they bound the generalisation loss averaged over {\em spherical} Gaussian perturbation, $\mathbb{E}_{\epsilon\sim\mathcal{N}(0,\rho^2 I)}
[l_D(\theta+\epsilon)]$. The latter might be an inaccurate measure for the average loss since the perturbation does not conform to the geometry of the underlying 
manifold. 
\end{remark}
The full proof is provided in Appendix~\ref{appsec:main_proof}, and we describe the proof sketch here. 
\begin{proof}[Proof (sketch)]
\vspace{-0.8em}
Our proof is an extension of~\cite{sam}'s proof, in which the PAC-Bayes bound~\cite{mcallester99} is considered for a pre-defined set of prior distributions, among which the one closest to the posterior is chosen to tighten the bound. 
In~\cite{sam}, the posterior is a spherical Gaussian (corresponding to a Euclidean  ball) with the variance being {\em independent} of the current model $\theta$. Then the prior set can be pre-defined as a series of spherical Gaussians with increasing variances so that there always exists a member in the prior set that matches the posterior by only small error. 
In our case, however, the posterior is a Gaussian with Fisher-induced ellipsoidal covariance, thus covariance being {\em dependent} on the current $\theta$. This implies that the prior set needs to be pre-defined more sophisticatedly to adapt to a not-yet-seen posterior. Our key idea is to partition the model parameter space $\Theta$ into many small Fisher ellipsoids 
$R_j \triangleq \{\theta \ \vert \  (\theta-\overline{\theta_j})^\top F(\overline{\theta_j})(\theta-\overline{\theta_j}) \leq r_j^2 \}, j=1,\dots,J$, for some fixed points $\{\overline{\theta_j}\}$, and we define the priors to be aligned with these ellipsoids. Then it can be shown that under some regularity conditions, any Fisher-induced ellipsoidal covariance of the posterior can match one of the $R_j$'s with small error, thus tightening the bound. 
\end{proof}

\subsection{Fisher SAM Illustration: Toy 2D Experiments}\label{sec:toy_expmt}

We devise a synthetic setup with 2D parameter space to illustrate the merits of the proposed FSAM against previous SAM and ASAM.
The model we consider is a univariate Gaussian, $p(x;\theta) = \mathcal{N}(x; \mu, \sigma^2)$ where $\theta = (\mu,\sigma) \in \mathbb{R} \times \mathbb{R}_{+} \subset \mathbb{R}^2$. For the loss function, we aim to build a one with two local minima, one with sharp curvature, the other flat. We further confine the loss to be a function of the model likelihood $p(x;\theta)$ so that the the parameter space becomes a manifold with the Fisher information metric. To this end, we define the loss function as a negative log-mixture of two KL-driven energy models. More specifically, 
\begin{align}
& l(\theta) = -\log \Big( \alpha_1 e^{-E_1(\theta)/\beta_1^2} + \alpha_2  e^{-E_2(\theta)/\beta_2^2} \Big), \label{eq:toy_loss} \\
& \ \ \ \ \ \ \textrm{where} \ \ E_i(\theta) = \textrm{KL}\big( p(x; \theta) || N(x; m_i, s_i^2) \big), \ i=1,2. \nonumber 
\end{align}
We set constants as: 
$(m_1,s_1,\alpha_1,\beta_1)=(20,30,0.7,1.8)$ and $(m_2,s_2,\alpha_2,\beta_2)=(-20,10,0.3,1.2)$. 
Since $\beta_i$ determines the component scale, we can guess that the flat minimum is at around $(m_1,s_1)$ (larger $\beta_1$), and the sharp one at around $(m_2,s_2)$ (smaller $\beta_2$). The contour map of $l(\theta)$ is depicted in Fig.~\ref{fig:toy_loss}, where the two minima numerically found are: $\theta^{flat} = (19.85,29.95)$ and 
$\theta^{sharp} = (-15.94,13.46)$ as marked in the figure. 
We prefer the flat minimum (marked as star/blue) to the sharp one (dot/magenta) even though $\theta^{sharp}$ attains slightly lower loss.

Comparing the neighborhood structures at the current iterate $(\mu,\sigma)$, 
SAM has a circle, $\{(\epsilon_1,\epsilon_2) \ | \ \epsilon_1^2 + \epsilon_2^2 \leq \gamma^2\}$, whereas FSAM has an ellipse, $\{(\epsilon_1,\epsilon_2) \ | \ \epsilon_1^2/\sigma^2 + \epsilon_2^2/(\sigma^2/2) \leq \gamma^2\}$ since the Fisher information for Gaussian is $F(\mu,\sigma) = \textrm{Diag}(1/\sigma^2, 2/\sigma^2)$. Note that the latter is the intrinsic metric for the underlying parameter manifold. Thus when $\sigma$ is large (away from 0), it is a valid strategy to explore more aggressively to probe the worst-case loss in both axes (as FSAM does). On the other hand, 
SAM is unable to adapt to the current iterate 
and probes relatively too little, 
which hinders finding a sensible robust loss function. This scenario is illustrated in Fig.~\ref{fig:toy_sam_fails}. The initial iterate (diamond/green) has a large $\sigma$ value, and FSAM makes aggressive exploration in both axes, helping us move toward the flat minimum. However, SAM makes too narrow exploration, merely converging to the relatively nearby sharp minimum.
For ASAM, the neighborhood at current iterate $(\mu,\sigma)$ is the magnitude-scaled ellipse, $\{(\epsilon_1,\epsilon_2) \ | \ \epsilon_1^2/\mu^2 + \epsilon_2^2/\sigma^2 \leq \gamma^2\}$. Thus when $\mu$ is close to 0, for instance, $\epsilon_1$ is not allowed to perturb much, hindering effective exploration of the parameter space toward robustness, as illustrated in Fig.~\ref{fig:toy_asam_fails}. 

\begin{figure}
\begin{center}
%
\centering
\includegraphics[trim = 2mm 2mm 2mm 2mm, clip, scale=0.314]{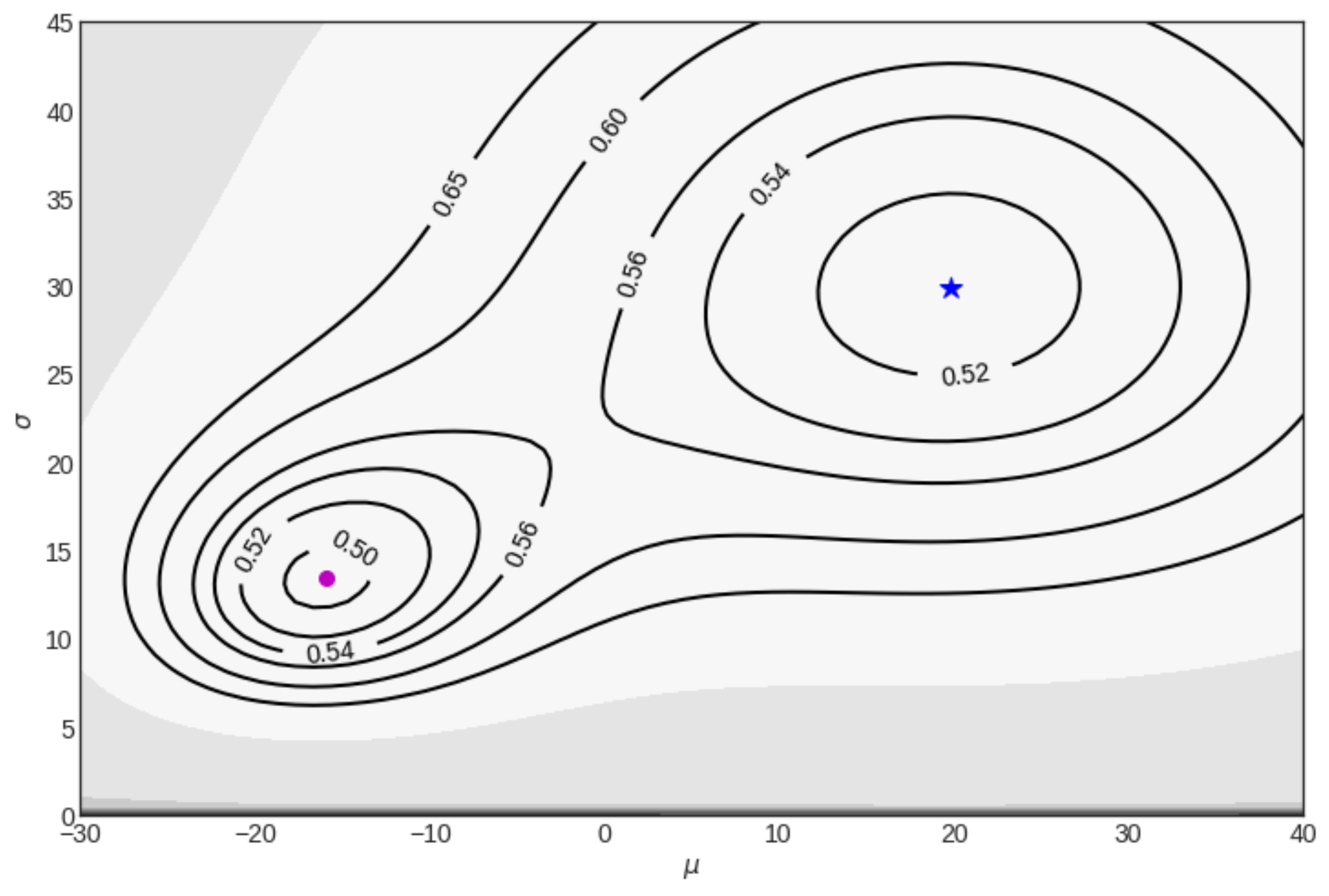}
\end{center}
\vspace{-1.8em}
\caption{(Toy experiment) Contour plot of the loss function. The flat minimum is shown as star/blue $\theta^{flat} = (19.85,29.95)$ ($l=0.51$, $H=0.001$), and the sharp one as circle/magenta $\theta^{sharp} = (-15.94,13.46)$ ($l=0.49$, $H=0.006$), with their loss values and Hessian traces shown in parentheses. 
}
\vspace{-1.5em}
\label{fig:toy_loss}
\end{figure}

\begin{figure*}
\begin{center}
%
\centering
\includegraphics[trim = 2mm 2mm 2mm 2mm, clip, scale=0.331]{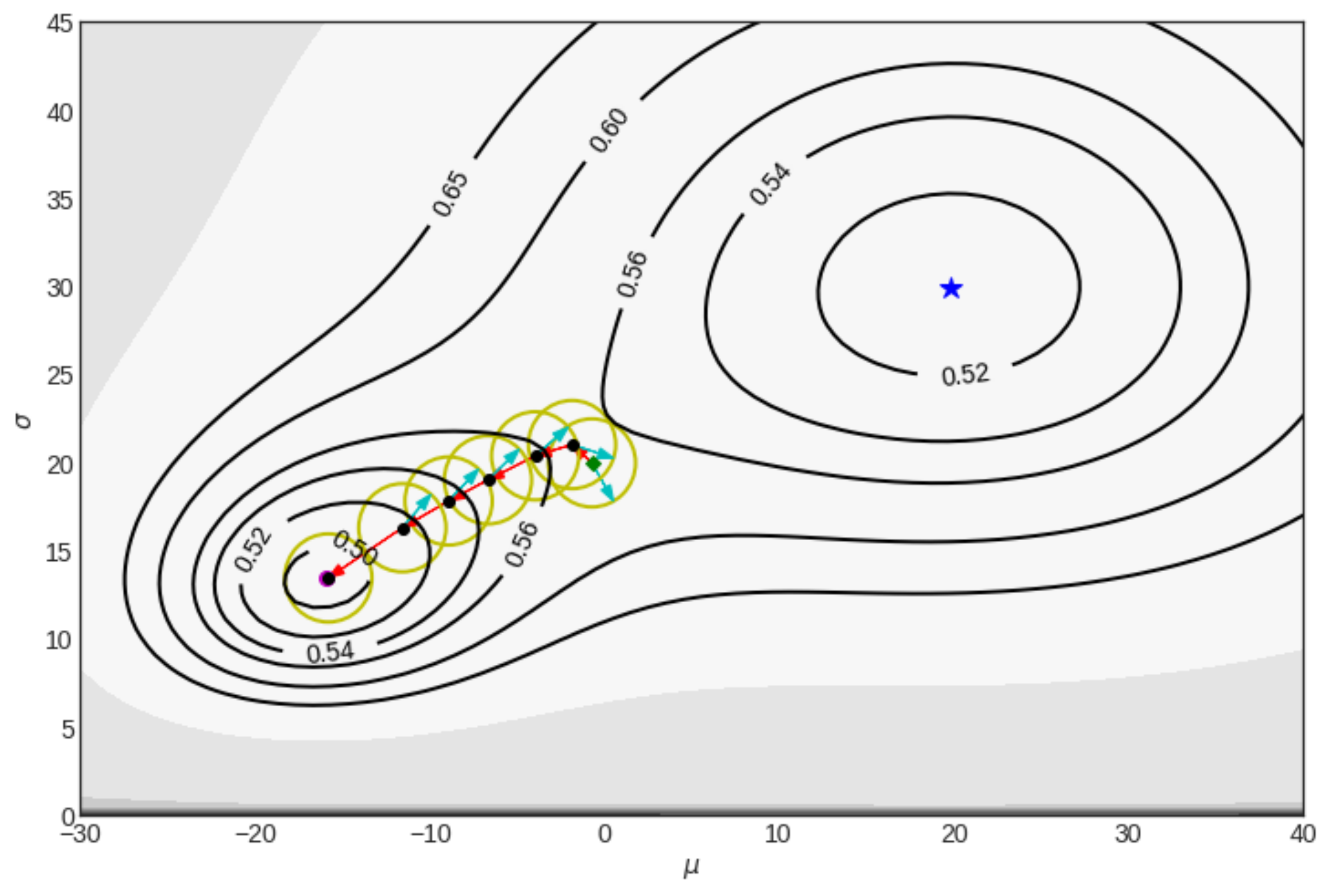} \ \
\includegraphics[trim = 2mm 2mm 2mm 2mm, clip, scale=0.331]{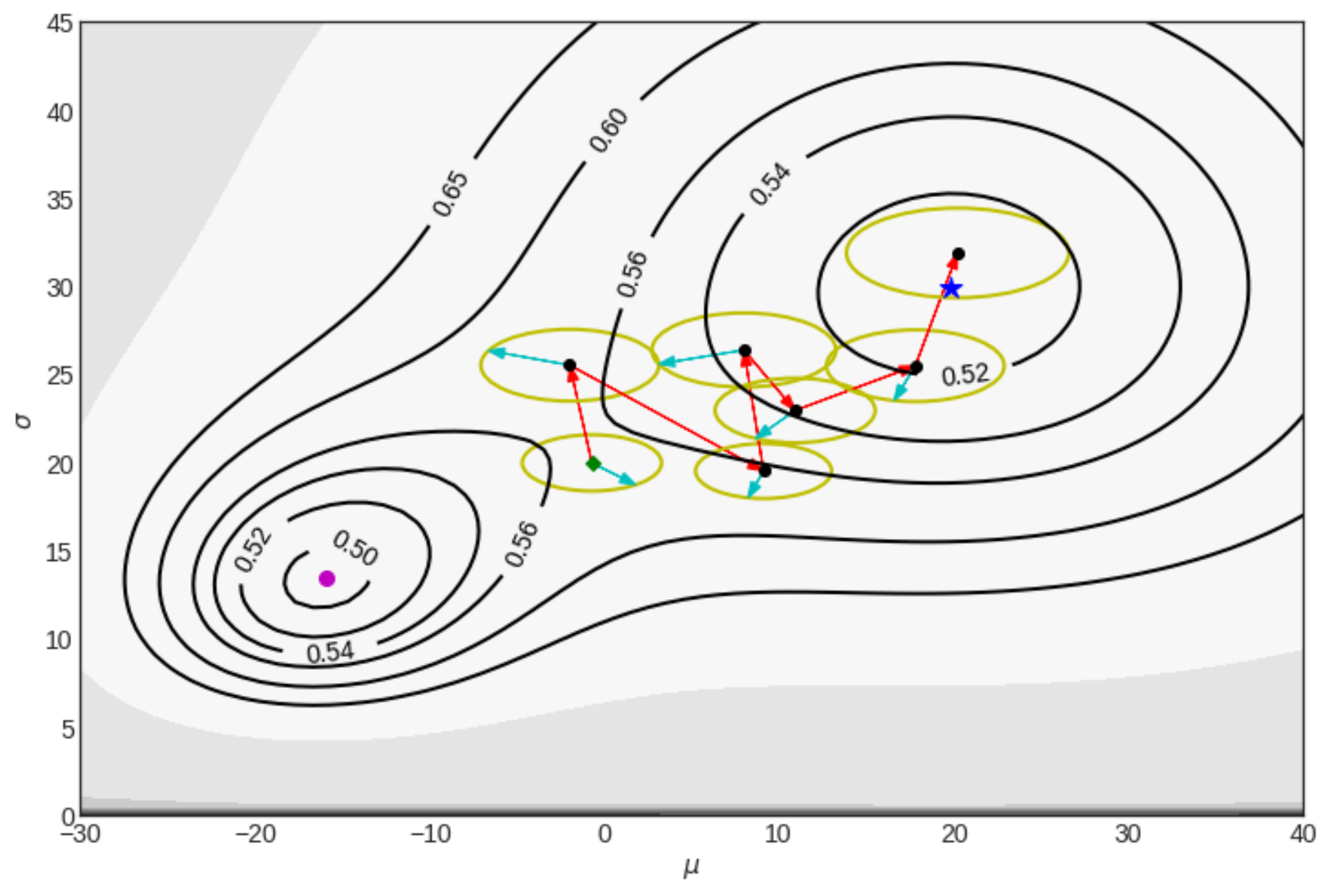}
\end{center}
\vspace{-1.0em}
\caption{(Toy experiment) \textbf{SAM vs.~FSAM}. X-axis is $\mu$ and Y-axis is $\sigma$.  (\textbf{Left$=$SAM}) SAM failed due to the inaccurate neighborhood structure of Euclidean ball. (\textbf{Right$=$FSAM}) FSAM finds the flat minimum due to the accurate neighborhood structure from Fisher information metric. 
Initial iterate shown as diamond/green; the neighborhood ball is depicted as yellow circle/ellipse; the worst-case probe within the neighborhood is indicated by cyan arrow, update direction is shown as red arrow. The sizes of circles/ellipses are adjusted for better visualisation. 
}
\label{fig:toy_sam_fails}
\end{figure*}
%
\begin{figure*}
\begin{center}
%
\centering
\includegraphics[trim = 2mm 2mm 2mm 2mm, clip, scale=0.331]{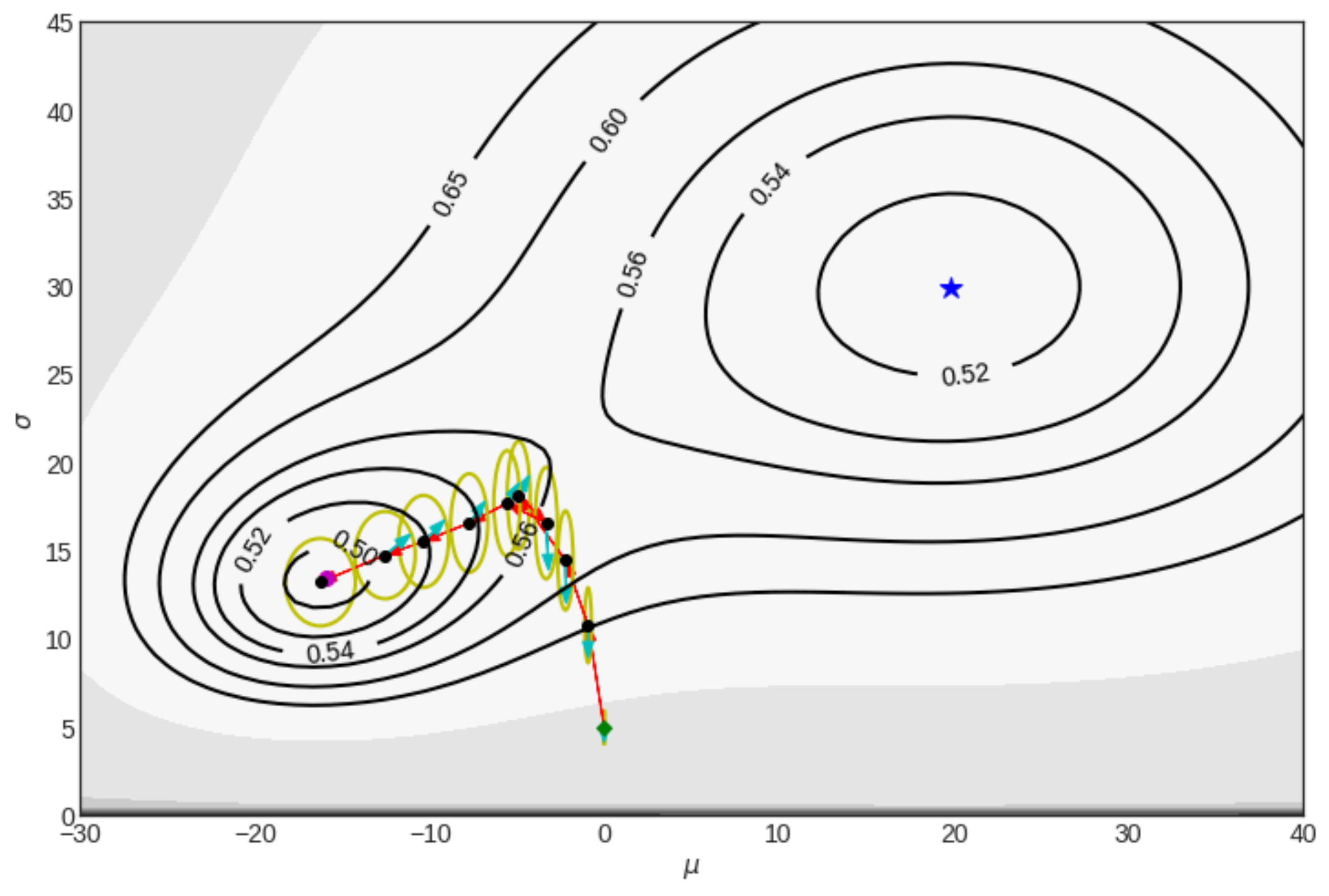} \ \
\includegraphics[trim = 2mm 2mm 2mm 2mm, clip, scale=0.331]{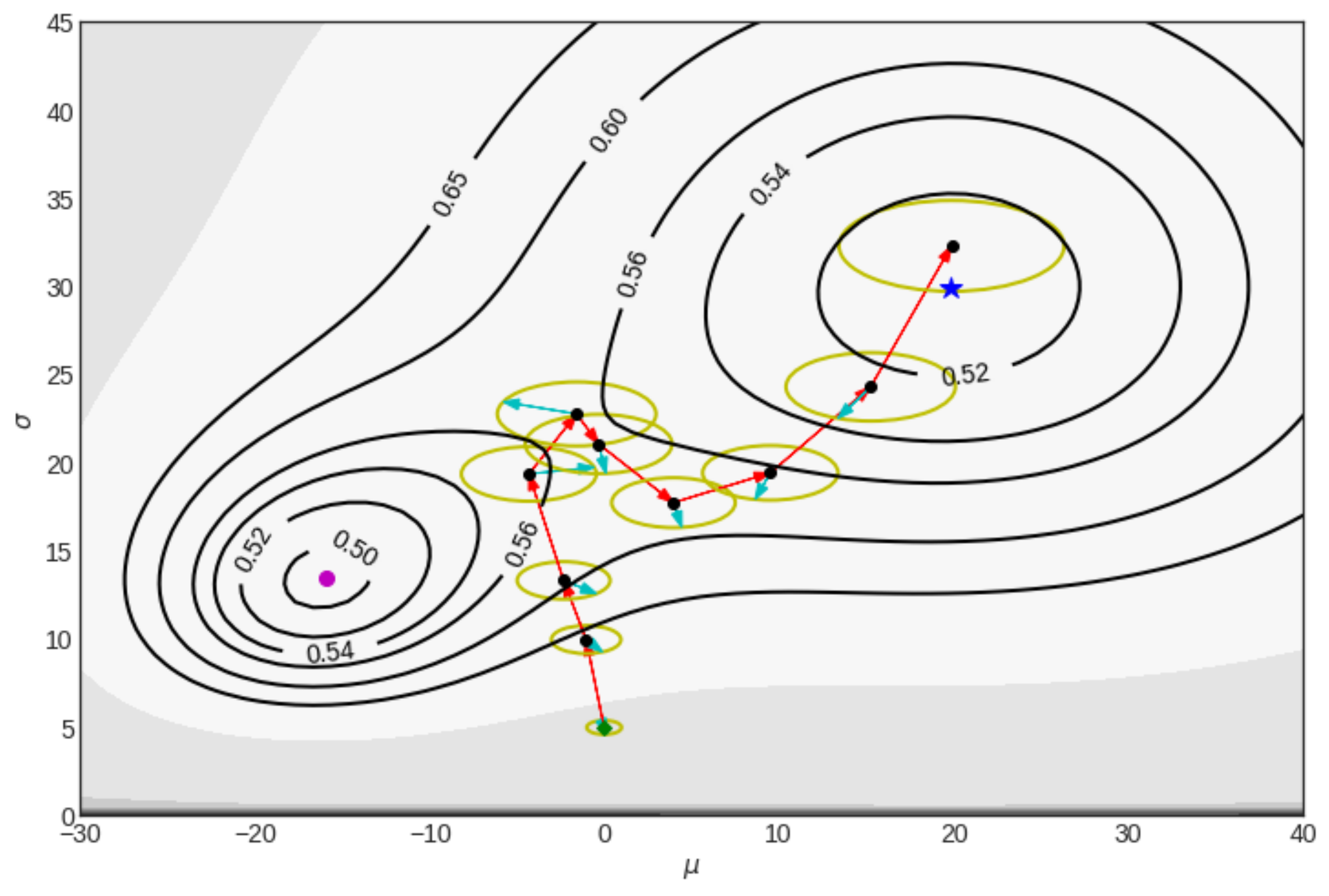}
\end{center}
\vspace{-1.0em}
\caption{(Toy experiment) \textbf{ASAM vs.~FSAM}. X-axis is $\mu$ and Y-axis is $\sigma$.  (\textbf{Left$=$ASAM}) ASAM failed due to the inaccurate neighborhood structure. Especially when the magnitude of a particular parameter value, in this case $\mu$, is small (close to 0), it overly penalises perturbation along the axis. The parameter $\mu$ has small magnitude initially, which forms an incorrect neighborhood ellipse overly shrunk along the X-axis, preventing it from finding an effective worst-case probe through X-axis perturbation. 
(\textbf{Right$=$FSAM}) FSAM finds the flat minimum due to the accurate neighborhood structure from Fisher information metric. 
Initial iterate shown as diamond/green; the neighborhood ball is depicted as yellow circle/ellipse; the worst-case probe within the neighborhood is indicated by cyan arrow, update direction is shown as red arrow. 
}
\label{fig:toy_asam_fails}
\end{figure*}
%
\begin{table*}
\vspace{-1.0em}
\caption{Test accuracies on CIFAR-10 and CIFAR-100.
}
\vspace{+0.3em}
\centering
\begin{footnotesize}
\centering
\scalebox{0.95}{
\begin{tabular}{l|cccc|cccc}
\toprule
 & \multicolumn{4}{c|}{CIFAR-10} & \multicolumn{4}{c}{CIFAR-100} \\
\cline{2-9}
 & SGD\Tstrut & SAM & ASAM & FSAM & SGD & SAM & ASAM & FSAM 
\\ \hline
DenseNet-121\Tstrut & $91.83^{\pm 0.13}$ & $92.44^{\pm 0.28}$ & $92.70^{\pm 0.30}$ & $\pmb{92.81}^{\pm 0.17}$ & $71.26^{\pm 0.15}$ & $72.83^{\pm 0.01}$ & $73.10^{\pm 0.23}$ & $\pmb{73.15}^{\pm 0.33}$ \\ 
ResNet-20\Tstrut & $92.91^{\pm 0.13}$ & $92.99^{\pm 0.16}$ & $92.92^{\pm 0.15}$ & $\pmb{93.18}^{\pm 0.11}$ & $68.24^{\pm 0.34}$ & $68.61^{\pm 0.26}$ & $68.68^{\pm 0.11}$ & $\pmb{69.04}^{\pm 0.30}$ \\ 
ResNet-56\Tstrut & $95.37^{\pm 0.06}$ & $95.59^{\pm 0.14}$ & $95.63^{\pm 0.07}$ & $\pmb{95.71}^{\pm 0.08}$ & $75.52^{\pm 0.27}$ & $76.44^{\pm 0.26}$ & $76.32^{\pm 0.14}$ & $\pmb{76.86}^{\pm 0.16}$ \\ 
VGG-19-BN\Tstrut & $95.70^{\pm 0.09}$ & $96.11^{\pm 0.09}$ & $95.97^{\pm 0.10}$ & $\pmb{96.17}^{\pm 0.07}$ & $73.45^{\pm 0.32}$ & $77.25^{\pm 0.24}$ & $74.36^{\pm 0.19}$ & $\pmb{77.86}^{\pm 0.22}$ \\ 
ResNeXt-29-32x4d\Tstrut & $96.55^{\pm 0.15}$ & $97.27^{\pm 0.10}$ & $97.29^{\pm 0.06}$ & $\pmb{97.47}^{\pm 0.05}$ & $79.36^{\pm 0.19}$ & $82.63^{\pm 0.16}$ & $82.41^{\pm 0.31}$ & $\pmb{82.92}^{\pm 0.15}$ \\ 
WRN-28-2\Tstrut & $95.56^{\pm 0.22}$ & $96.28^{\pm 0.14}$ & $96.25^{\pm 0.07}$ & $\pmb{96.51}^{\pm 0.08}$ & $78.85^{\pm 0.25}$ & $79.87^{\pm 0.13}$ & $80.17^{\pm 0.14}$ & $\pmb{80.22}^{\pm 0.26}$ \\ 
WRN-28-10\Tstrut & $97.12^{\pm 0.10}$ & $97.56^{\pm 0.06}$ & $97.63^{\pm 0.04}$ & $\pmb{97.89}^{\pm 0.07}$ & $83.47^{\pm 0.21}$ & $\pmb{85.60}^{\pm 0.05}$ & $85.20^{\pm 0.18}$ & $\pmb{85.60}^{\pm 0.11}$ \\ \hline
PyramidNet-272\Tstrut & $97.73^{\pm 0.04}$ & $97.91^{\pm 0.02}$ & $97.91^{\pm 0.01}$ & $\pmb{97.93}^{\pm 0.04}$ & $83.46^{\pm 0.02}$ & $85.19^{\pm 0.04}$ & $85.05^{\pm 0.11}$ & $\pmb{86.93}^{\pm 0.14}$ \\ 
\bottomrule
\end{tabular}
}
\end{footnotesize}
\label{tab:cifar}
\end{table*}

\section{Related Work}\label{sec:related}
Interest in flat minima for neural networks dates back to at least  \cite{flat_minima95,flat_minima97}, who characterise flatness as the size of the region around which the training loss remains low. Many studies have since investigated the link between flat minima and generalisation performance  \cite{keskar2017batchGen,Neyshabur17,dr17,Dinh17}.
In particular, the correlation between sharpness and generalisation performance was studied with diverse measures empirically on large-scale experiments: \cite{jiang19}. 
Beyond the IID setting, \cite{cha2021swad} analyse the impact of flat minima on generalisation performance under domain-shift. 

Motivated by these analyses, several methods have been proposed to visualise and optimise for flat minima, with the aim of improving generalisation. \cite{hao2018lossLandscape} developed methods for visualising loss surfaces to inspect minima shape. \cite{zhu2019anisotropic} analyse how the noise in SGD promotes preferentially discovering flat minima over sharp minima, as a potential  explanation for SGD's strong generalisation. 
Entropy-SGD \cite{chaudhar2017entropySGD} biases gradient descent toward flat minima by regularising local entropy. Stochastic Weight Averaging (SWA) \cite{izmailov2018averaging} was proposed as an approach to find flatter minima by parameter-space averaging of an ensemble of models' weights. The state-of-the-art SAM \cite{sam} and ASAM \cite{asam} find flat minima by reporting the worst-case loss within a ball around the current iterate. Our Fisher SAM builds upon these by providing a principled approach to defining a non-Euclidean ball within which to compute the worst-case loss.


\section{Experiments}\label{sec:expmt}

We empirically demonstrate generalisation performance (Sec.~\ref{sec:image_classify}--\ref{sec:transfer}) and noise robustness (Sec.~\ref{sec:adversarial},~\ref{sec:label_noise}) of the proposed Fisher SAM method. As competing approaches, we consider the vanilla (non-robust) optimisation (\textbf{SGD}) as a baseline, and two latest SAM approaches: \textbf{SAM}~\cite{sam} that uses Euclidean-ball neighborhood and \textbf{ASAM}~\cite{asam} that employs  parameter-scaled neighborhood. Our approach, forming Fisher-driven neighborhood, is denoted by \textbf{FSAM}. 

For the implementation of Fisher SAM in the experiments, instead of simply adding a  regulariser to each diagonal entry $f_i$ of the Fisher information matrix $\hat{F}(\theta)$, we take $1/(1+\eta f_i)$ as the diagonal entry of the inverse Fisher. Hence $\eta$ serves as anti-regulariser (e.g., small $\eta$ diminishes or regularises the Fisher impact). We find this implementation performs better than simply adding a regulariser. In most of our experiments, we set $\eta=1.0$. 
Certain multi-GPU/TPU gradient averaging heuristics called the $m$-sharpness trick empirically improves the generalisation performance of SAM and ASAM~\cite{sam}. However, since the trick is theoretically less justified, we do not use the trick in our experiments for fair comparison. 


\subsection{Image Classification}\label{sec:image_classify}

The goal of this section is to empirically compare generalisation performance of the competing algorithms: SGD, SAM, ASAM, and our FSAM on image classification problems. Following the experimental setups suggested in~\cite{sam,asam}, we employ several ResNet~\cite{resnet}-based backbone networks including WideResNet~\cite{wrn}, VGG~\cite{vgg}, DenseNet~\cite{densenet}, ResNeXt~\cite{resnext}, and PyramidNet~\cite{pyramidnet}, on the CIFAR-10/100 datasets~\cite{cifar}.
Similar to~\cite{sam,asam}, we use the SGD optimiser with momentum 0.9, weight decay 0.0005, initial learning rate 0.1, cosine learning rate scheduling~\cite{coslr}, for up to 200 epochs (400 for SGD) with batch size 128. For the PyramidNet, we use batch size 256, initial learning rate $0.05$ trained up to 900 epochs (1800 for SGD). 
We also apply Autoaugment~\cite{autoaug},  Cutout~\cite{cutout} data augmentation, and the label smoothing~\cite{labsm} with factor $0.1$ is used for defining the loss function. 

We perform the grid search to find best hyperparameters $(\gamma,\eta)$ for FSAM, and they are $(\gamma=0.1,\eta=1.0)$ for both CIFAR-10 and CIFAR-100 across all backbones except for PyramidNet.  For the PyramidNet on CIFAR-100, we set $(\gamma=0.5,\eta=0.1)$. For SAM and ASAM, we follow the best hyperparameters reported in their papers: (SAM) $\gamma=0.05$ and (ASAM) $\gamma=0.5,\eta=0.01$ for CIFAR-10 and (SAM) $\gamma=0.1$ and (ASAM) $\gamma=1.0,\eta=0.1$ for CIFAR-100. For the PyramidNet, (SAM) $\gamma=0.05$ and (ASAM) $\gamma=1.0$.
The results are summarized in Table~\ref{tab:cifar}, where Fisher SAM consistently outperforms SGD and previous SAM approaches for all backbones. This can be attributed to FSAM's correct neighborhood estimation that respects the underlying geometry of the parameter space.

\subsection{Extra (Over-)Training on ImageNet}\label{sec:imagenet}

For a large-scale experiment, we consider the ImageNet dataset~\cite{deng2009imagenet}.  
We use the DeiT-base (denoted by DeiT-B) vision transformer model~\cite{pmlr-v139-touvron21a} as a powerful backbone network. Instead of training the DeiT-B model from the scratch, we use the publicly available\footnote{\url{https://github.com/facebookresearch/deit}} ImageNet pre-trained parameters as a warm start, and perform finetuning with the competing loss functions. Since the same dataset is used for pre-training and finetuning, it can be better termed extra/over-training. 

The main goal of this experimental setup is to see if robust sharpness-aware loss functions in the extra training stage can further improve the test performance. 
First, we measure the test performance of the pre-trained DeiT-B model, which is $81.94\%$ (Top-1) and $95.63\%$ (Top-5). After three epochs of extra training, the test accuracies of the competing approaches are summarized in Table~\ref{tab:imagenet}. For extra training, we use hyperparameters: SAM ($\gamma=0.05$), ASAM ($\gamma=1.0,\eta=0.01$), and FSAM ($\gamma=0.5,\eta=0.1$) using the SGD optimiser with batch size 256, weight decay $0.0001$, initial learning rate $10^{-5}$ and the cosine scheduling. 
Although the improvements are not very drastic, the sharpness-aware loss functions appear to move the pre-trained model further toward points that yield better generalisation performance, and our FSAM attains the largest improvement among other SAM approaches. 

\begin{table}
\vspace{-1.0em}
\caption{Extra-training results (test accuracies) on ImageNet. Before extra-training starts, $81.94\%$ (Top-1) and $95.63\%$ (Top-5).
}
\vspace{+0.3em}
\centering
\begin{footnotesize}
\centering
\scalebox{0.95}{
\begin{tabular}{c|cccc}
\toprule
 & SGD\Tstrut & SAM & ASAM & FSAM \\ \midrule 
Top-1\Tstrut & $81.97^{\pm 0.01}$ & $81.99^{\pm 0.01}$ & $82.03^{\pm 0.04}$ & $\pmb{82.17}^{\pm 0.01}$\\ \hline
Top-5\Tstrut & $95.61^{\pm 0.01}$ & $95.71^{\pm 0.03}$ & $95.83^{\pm 0.04}$ & $\pmb{95.90}^{\pm 0.01}$ \\ 
\bottomrule
\end{tabular}
}
\end{footnotesize}
\label{tab:imagenet}
\end{table}

\subsection{Transfer Learning by Finetuning}\label{sec:transfer}

One of the powerful features of the deep neural network models trained on extremely large datasets, is the transferability, that is, the models tend to adapt easily and quickly to novel target datasets and/or downstream tasks by finetuning. 
We use the vision transformer model ViT-base with $16\times 16$ patches (ViT-B/16)~\cite{vit} pre-trained on the ImageNet (with publicly available checkpoints), and finetune the model on the target datasets: CIFAR-100, Stanford Cars~\cite{cars}, and Flowers~\cite{flowers}. 
We run SGD, SAM ($\gamma=0.05$), ASAM ($\gamma=0.1,\eta=0.01$), and FSAM ($\gamma=0.1,\eta=1.0$) with the SGD optimiser for 200 epochs, batch size 256, weight decay $0.05$, initial learning rate $0.0005$ and the cosine scheduling. As the results in Table~\ref{tab:transfer} indicate, FSAM consistently improves the performance over the competing methods. 

\begin{table}
\vspace{-1.4em}
\caption{Test accuracies for transfer learning. 
}
\vspace{+0.3em}
\centering
\begin{footnotesize}
\centering
\scalebox{0.95}{
\begin{tabular}{c|cccc}
\toprule
 & SGD\Tstrut & SAM & ASAM & FSAM \\ \midrule 
\scriptsize{CIFAR}\Tstrut & $87.97^{\pm 0.12}$ & $87.99^{\pm 0.09}$ & $87.97^{\pm 0.08}$ & $\pmb{88.39}^{\pm 0.13}$ \\
\scriptsize{Cars}\Tstrut & $92.85^{\pm 0.31}$ & $93.29^{\pm 0.01}$ & $93.28^{\pm 0.02}$ & $\pmb{93.42}^{\pm 0.01}$ \\ 
\scriptsize{Flowers}\Tstrut & $94.53^{\pm 0.20}$ & $95.05^{\pm 0.06}$ & $95.08^{\pm 0.10}$ & $\pmb{95.26}^{\pm 0.03}$ \\ 
\bottomrule
\end{tabular}
}
\end{footnotesize}
\label{tab:transfer}
\vspace{-0.8em}
\end{table}


\subsection{Robustness to Adversarial Parameter Perturbation}\label{sec:adversarial}

Another important benefit of the proposed approach is robustness to parameter perturbation. In the literature, the generalisation performance of the corrupted models is measured by injecting artificial noise to the learned parameters, which serves as a measure of vulnerability of neural networks~\cite{chen17,arora18,dai19,gu19,nagel19,rakin20,weng20}. 
Although it is popular to add Gaussian random noise to the parameters, recently the {\em adversarial} perturbation~\cite{adv_def} was proposed where they consider the worst-case scenario under parameter corruption, which amounts to perturbation along the gradient direction. More specifically, the perturbation process is: $\theta \to \theta + \alpha \frac{\nabla l(\theta)}{\|\nabla l(\theta)\|}$ where $\alpha>0$ is the perturbation strength that can be chosen. It turns out to be a more effective perturbation measure than the random (Gaussian noise) corruption. 

We apply this adversarial parameter perturbation process to ResNet-34 models trained by SGD, SAM ($\gamma=0.05$), and FSAM ($\gamma=0.1,\eta=1.0$) on CIFAR-10, where we vary the perturbation strength $\alpha$ from 0.1 to 5.0. The results are depicted in Fig.~\ref{fig:param_perturb}. While we see performance drop for all models as $\alpha$ increases, eventually reaching nonsensical models (pure random prediction accuracy $10\%$) after $\alpha\geq 5.0$, the proposed Fisher SAM exhibits the least performance degradation among the competing methods, proving the highest robustness to the adversarial parameter corruption. 

\begin{figure}
\begin{center}
%
\centering
\includegraphics[trim = 5mm 4mm 5mm 5mm, clip, scale=0.335]{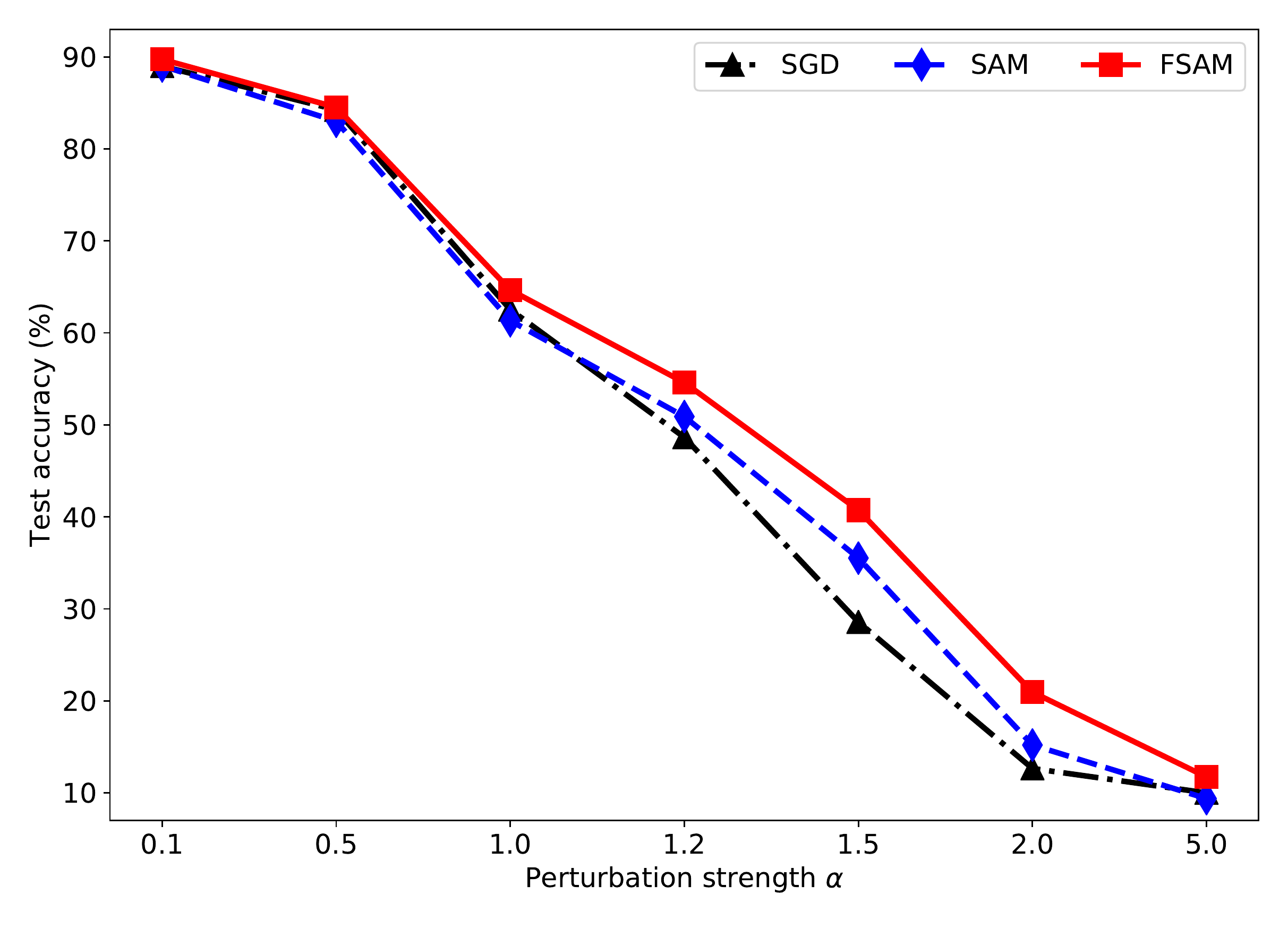}
\end{center}
\vspace{-1.8em}
\caption{Adversarial parameter perturbation.
}
\vspace{-0.8em}
\label{fig:param_perturb}
\end{figure}



\subsection{Robustness to Label Noise }\label{sec:label_noise}

In the previous works, SAM and ASAM are shown to be robust to label noise in training data. Similarly as their experiments, we introduce symmetric label noise by random flipping with corruption levels 20, 40, 60, and $80\%$, as introduced in~\cite{rooyen15}. The results on ResNet-32 networks on the CIFAR-10 dataset are summarized in Table~\ref{tab:label_noise}. Our Fisher SAM shows high robustness to label noise comparable to SAM while exhibiting significantly large improvements over SGD and ASAM. 
%
\begin{table}
\vspace{-0.1em}
\caption{Label noise. Test accuracies on CIFAR-10. 
The hyperparameters are: (SAM) $\gamma=0.1$, (ASAM) $\gamma=0.5,\eta=0.01$, and (FSAM) $\gamma=0.1,\eta=0.001$.
}
\vspace{+0.3em}
\centering
\begin{footnotesize}
\centering
\scalebox{0.95}{
\begin{tabular}{c|cccc}
\toprule
Rates\Tstrut & SGD & SAM & ASAM & FSAM \\ \midrule 
0.2 &$87.97^{\pm 0.04}$ &$\pmb{93.12}^{\pm 0.24}$ &$92.26^{\pm 0.33}$ & $93.03^{\pm 0.11}$\\ \hline
0.4\Tstrut &$83.60^{\pm 0.59}$ &$90.54^{\pm 0.19}$ &$88.47^{\pm 0.06}$ & $\pmb{90.95}^{\pm 0.17}$ \\ \hline
0.6\Tstrut &$76.97^{\pm 0.31}$ &$85.39^{\pm 0.52}$ &$82.32^{\pm 0.55}$ & $\pmb{85.76}^{\pm 0.21}$ \\ \hline
0.8\Tstrut &$66.32^{\pm 0.27}$ &$74.31^{\pm 1.02}$ &$70.56^{\pm 0.27}$ & $\pmb{74.66}^{\pm 0.67}$ \\
\bottomrule
\end{tabular}
}
\end{footnotesize}
\label{tab:label_noise}
\vspace{-1.0em}
\end{table}

\subsection{Hyperparameter Sensitivity}\label{sec:hparam}

In our Fisher SAM, there are two hyperparameters: $\gamma=$ the size of the neighborhood and $\eta=$ the anti-regulariser for the Fisher impact. We demonstrate the sensitivity of Fisher SAM to these hyperparameters. To this end, we train WRN-28-10 backbone models trained with the FSAM loss on the CIFAR-100 dataset for different hyperparameter combinations: $(\gamma,\eta) \in \{0.01,0.05,0.1,0.5,1.0\} \times \{10^{-4},10^{-3},10^{-2},10^{-1},1.0,10\}$. In Fig.~\ref{fig:hparams} we plot the test accuracy of the learned models\footnote{Note that there are discrepancies from Table~\ref{tab:cifar} that may arise from the lack of data augmentation.}. The results show that unless $\gamma$ is chosen too large (e.g., $\gamma=1.0$), the learned models all perform favorably well, being less sensitive to the hyperparameter choice. But the best performance is attained when $\gamma$ lies in between $0.1$ and $0.5$, with some moderate values for the Fisher impact $\eta$ in between $0.01$ and $1.0$. 

\begin{figure}[t!]
\begin{center}
%
\centering
\includegraphics[trim = 5mm 4mm 6mm 5mm, clip, scale=0.335]{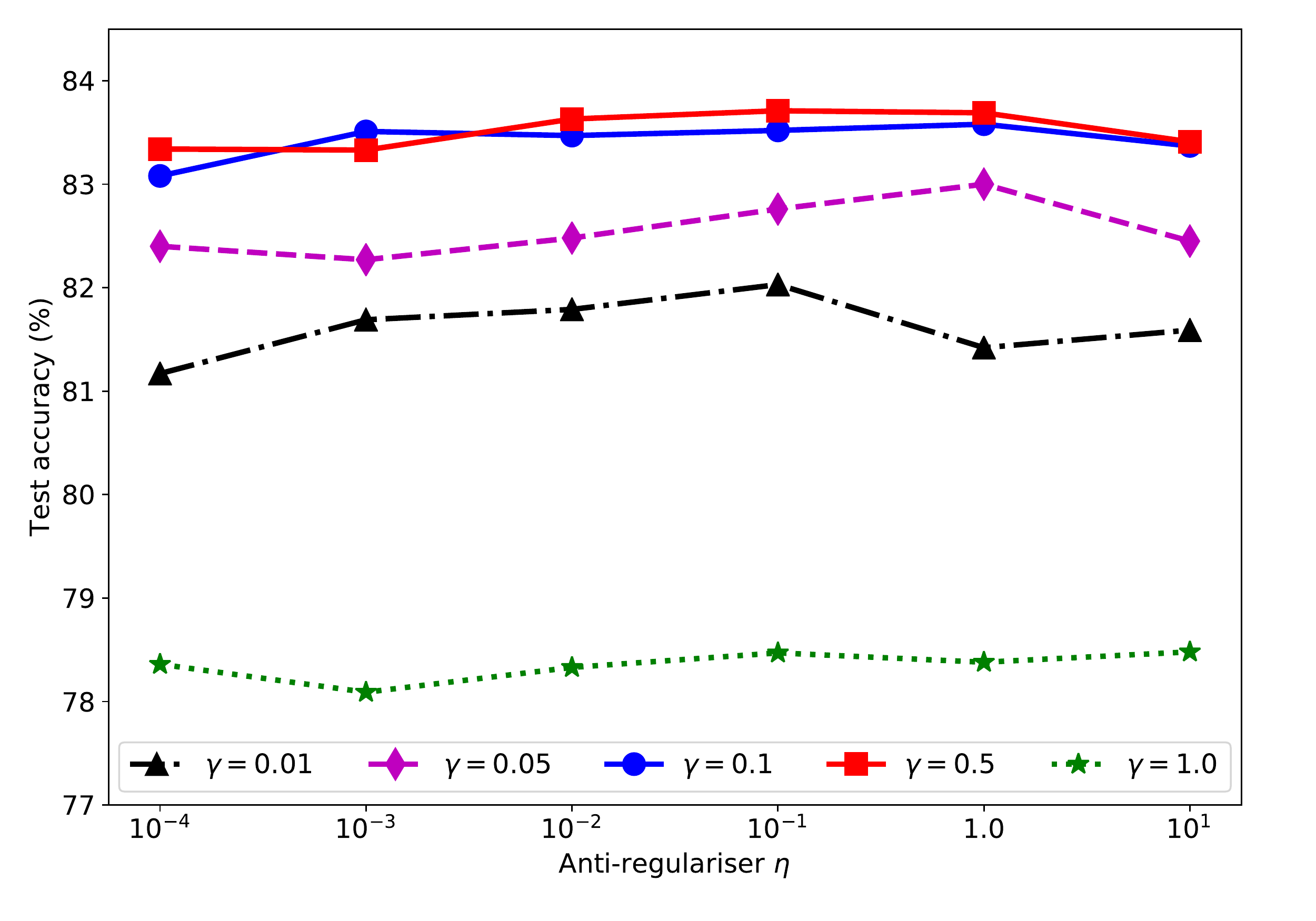}
\end{center}
\vspace{-1.8em}
\caption{Hyperparameter sensitivity of Fisher SAM. 
}
\label{fig:hparams}
\end{figure}

\subsection{Computational Overhead of FSAM}

Compared to SAM, our FSAM requires only extra cost of element-wise vector product under our diagonal gradient-magnitude approximation schemes. In practice, the difference is negligible: the per-batch (batch size 128) times for CIFAR10/WRN28-10 are: 0.2322 seconds (SAM), 0.2334 seconds (FSAM) on a single RTX 2080 Ti machine.

\section{Conclusion}\label{sec:conclusion}

In this paper we have proposed a novel sharpness aware loss function that incorporates the information geometry of the underlying parameter manifold, which defines a more accurate intrinsic neighborhood structure, addressing the issues of the previous flat-minima optimisation methods. The proposed algorithm remains computationally efficient via a theoretically justified gradient magnitude approximation for the Fisher information matrix. By proving the theoretical generalisation bound, and through diverse experiments on image classification, extra-training/finetuning, data corruption, and model perturbation, we demonstrated improved generalisation performance and robustness of the proposed Fisher SAM. 
Several research questions remain: natural gradient updates combined with the Fisher SAM loss, investigation of distributed update averaging schemes ($m$-sharpness), and discovering the relation to the proximal gradients, among others, and we leave them as future work.


\bibliography{main}
\bibliographystyle{icml2022}

\newpage
\appendix
\onecolumn


\section{Theorem~\ref{thm:main} and  Proof}\label{appsec:main_proof}

We first state several regularity conditions and  assumptions under which the theorem can be proved formally. 
\begin{assumption}\label{a1}
Let $k$ be the dimensionality of the model parameters $\theta$. We consider the following $J$ ellipsoids centered at some fixed points $\overline{\theta_j} \in \mathbb{R}^k$ with elliptic axes determined by the Fisher information $F(\overline{\theta_j})$ and sizes $r_j$: 
\begin{align}
R_j \triangleq \{\theta \in \mathbb{R}^k \ \vert \  (\theta-\overline{\theta_j})^\top F(\overline{\theta_j})(\theta-\overline{\theta_j}) \leq r_j^2 \}, \ \ \ \ j=1,\dots,J.
\label{appeq:ellipsoids}
\end{align}
We choose $\overline{\theta_j}$ properly such that $F(\overline{\theta_j})$ are strictly positive\footnote{The strict positive definiteness of the Fisher can be assured for non-redundant parametrisation, and can be mildly assumed.} (all eigenvalues greater than some constant $\lambda_{min}>0$).
Our model parameter space $\Theta$ is assumed to be contained in these ellipsoids, i.e., $\Theta \subseteq \cup_{j=1}^J R_j$. 
We also assume $\Theta$ has a bounded diameter $B$, that is, $B \geq \textrm{diam}(\Theta) = \max_{\theta,\theta'\in\Theta} \|\theta-\theta'\|$ (thus $\|\theta\|\leq B$). 
Note that since $\textrm{vol}(R_j) \propto r_j^k \cdot |F(\overline{\theta_j})|^{-1/2}$, we have $J = O(\max_j \textrm{diam}(\Theta)^k / r_j^k)$, and thus $\log J = O(k)$.
\end{assumption}
The following two assumptions are regularity conditions regarding smoothness of the Fisher information matrix $F(\theta)$ as a function of $\theta$, that are assumed to hold for $\theta$ in each ellipsoid $R_j$.  Intuitively these conditions can be met by adjusting $r_j$ sufficiently small, but specific conditions are provided below. 
\begin{assumption}\label{a2}
Let $\lambda_i(A)$ be the $i$-th largest eigenvalue of the (positive definite) matrix $A$. Then for $j=1,\dots,J$,
\begin{align}
\frac{\lambda_i(F(\theta))}{\lambda_i(F(\overline{\theta_j}))} = 1 + c_{ij}, \ \ c_{ij} \in[-\epsilon_{min},\epsilon_{min}], \ \ \forall\theta\in R_j, \ \forall i=1,\dots,k,
\label{appeq:a2}
\end{align}
where $\epsilon_{min}$ is a small positive constant. 
For instance, if $\lambda_i(\theta) \triangleq \lambda_i(F(\theta))$ is Lipschitz continuous with constant $C_1$, that is, 
\begin{align}
|\lambda_i(\theta)-\lambda_i(\overline{\theta_j})| \leq C_1 \|\theta-\overline{\theta_j}\|_{F(\overline{\theta_j})}, \ \ \forall\theta \in R_j
\end{align}
where $\|x\|_A = (x^\top A x)^{1/2}$, then (\ref{appeq:a2}) holds by adjusting $r_j$ properly. More specifically,
\begin{align}
\bigg|\frac{\lambda_i(\theta)}{\lambda_i(\overline{\theta_j})} - 1\bigg| = \lambda_i(\overline{\theta_j})^{-1} |\lambda_i(\theta)-\lambda_i(\overline{\theta_j})| \leq \lambda_i(\overline{\theta_j})^{-1} C_1 \|\theta-\overline{\theta_j}\|_{F(\overline{\theta_j})} \leq \lambda_i(\overline{\theta_j})^{-1} C_1 r_j.
\end{align}
Hence we can choose $r_j \leq \lambda_i(\overline{\theta_j}) \epsilon_{min}/C_1$ to make (\ref{appeq:a2}) hold.
\end{assumption}
\begin{assumption}\label{a3}
We assume that $F(\theta)$ is non-singular, and for $j=1,\dots,J$,
\begin{align}
F(\overline{\theta_j}) F(\theta)^{-1} = I + A^j, \ \ A^j_{i,i'}\in[-\epsilon_{min},\epsilon_{min}], \ \ \forall\theta\in R_j
\label{appeq:a3}
\end{align}
For instance, if $F(\theta)^{-1}$ is Lipschitz continuous with constant $C_2$, that is,
\begin{align}
\|F(\theta)^{-1}-F(\overline{\theta_j})^{-1}\| \leq C_2 \|\theta-\overline{\theta_j}\|_{F(\overline{\theta_j})}, \ \ \forall\theta \in R_j
\end{align}
where the matrix norm in the RHS is the max-norm, i.e., $\|B\| = \max_{i,i'} |B_{i,i'}|$, then we have
\begin{align}
\|F(\overline{\theta_j}) F(\theta)^{-1} - I\| \ &= \ \|F(\overline{\theta_j}) \big(  F(\theta)^{-1} - F(\overline{\theta_j})^{-1} \big)\| \ \leq \ \|F(\overline{\theta_j})\| \|F(\theta)^{-1}-F(\overline{\theta_j})^{-1}\| \\
&\leq \ \|F(\overline{\theta_j})\| C_2 \|\theta-\overline{\theta_j}\|_{F(\overline{\theta_j})} \ \leq \  \|F(\overline{\theta_j})\| C_2 r_j. 
\end{align}
Choosing $r_j \leq \|F(\overline{\theta_j})\|^{-1}  \epsilon_{min}/C_2$ is sufficient to satisfy (\ref{appeq:a3}).

\end{assumption}

Now we re-state our main theorem. 
\begin{theorem}[Generalisation bound of Fisher SAM] 
Let $\Theta\subseteq\mathbb{R}^k$ be the model parameter space as described in the above assumptions. 
For any $\theta\in\Theta$, with probability at least $1-\delta$ over the choice of the training set $S$ ($|S|=n$), 
\begin{align}
\mathbb{E}_{\epsilon\sim\mathcal{N}(0,\rho^2 F(\theta)^{-1})}
[l_D(\theta+\epsilon)] \ \leq \  l^\gamma_{FSAM}(\theta; S) + \sqrt{\frac{O(k + \log \frac{n}{\delta})}{n-1}},
\label{appeq:fsam_bound}
\end{align}
where $l_D(\cdot)$ is the generalisation loss, $l^\gamma_{FSAM}(\cdot; S) =  \max_{\epsilon^\top F(\theta) \epsilon \leq \gamma^2} l_S(\theta+\epsilon)$ is the empirical Fisher SAM loss on $S$, 
and $\rho = ( \sqrt{k} + \sqrt{\log n} )^{-1} \gamma$. 
\label{appthm:main}
\end{theorem}
\begin{proof}
Motivated from~\cite{sam}, we use the PAC-Bayes theorem~\cite{mcallester99} to derive the bound. According to the PAC-Bayes generalisation bound of~\cite{mcallester99,dr17}, for any prior distribution $P(\theta)$, with probability at least $1-\delta$ over the choice of the training set $S$, it holds that 
\begin{align}
\mathbb{E}_{Q(\theta)}[l_D(\theta)] \leq \mathbb{E}_{Q(\theta)}[l_S(\theta)] + \sqrt{\frac{\textrm{KL}(Q(\theta)||P(\theta)) + \log\frac{n}{\delta}}{2(n-1)}}
\label{appeq:pac_bayes}
\end{align}
for any posterior distribution $Q(\theta)$ that may be dependent on the training data $S$. In (\ref{appeq:pac_bayes}), $l_D(\cdot)$ and $l_S(\cdot)$ are generalisation and empirical losses, respectively. 
We choose $Q(\theta) = \mathcal{N}(\theta_0, \rho^2 F(\theta_0)^{-1})$, a Gaussian centered at $\theta_0$ with the covariance aligned with the Fisher information metric at $\theta_0$. One can choose $\theta_0$ arbitrarily from $\Theta$, and it can be dependent on $S$ (in which sense, a more accurate notation would be $\theta_{|S}$, however, we use $\theta_0$ for simplicity). To minimise the bound of (\ref{appeq:pac_bayes}), we aim to choose the prior $P(\theta)$ that minimises the KL divergence term, which coincides with $Q(\theta)$. However, this would violate the PAC-Bayes assumption where the prior should be independent of $S$. The idea, inspired by the covering approach~\cite{langford02,chatterji20,sam}, is to have a pre-defined set of (data independent) prior distributions for all of which the PAC-Bayes bounds hold, and we select the one that is closest to $Q(\theta)$ from the pre-defined set. 


Specifically, we define $J$ prior distributions $\{P_j(\theta)\}_{j=1}^J$ as $P_j(\theta) = \mathcal{N}(\overline{\theta_j}, \rho^2 F(\overline{\theta_j})^{-1})$ sharing the  centers and covariances with the ellipsoids defined as above. 
Then applying the PAC-Bayes bound (\ref{appeq:pac_bayes}) for each $j$ makes the following hold for $P_j(\theta)$ with probability at least $1-\delta_j$ over the choice of the training set $S$,
\begin{align}
\forall Q(\theta), \ \ \mathbb{E}_{Q(\theta)}[l_D(\theta)] \leq \mathbb{E}_{Q(\theta)}[l_S(\theta)] + \sqrt{\frac{\textrm{KL}(Q(\theta)||P_j(\theta)) + \log\frac{n}{\delta_j}}{2(n-1)}}. 
\label{appeq:pac_bayes_each_j}
\end{align}
By having the intersection of the training sets for which (\ref{appeq:pac_bayes_each_j}) holds, we can say that (\ref{appeq:pac_bayes_each_j}) holds for {\em all} $P_j(\theta)$ ($\forall j=1,\dots,J$) over the intersection. By the union bound theorem, the probability over the choice of the intersection is at least $1-\sum_{j=1}^J \delta_j$. By letting $\delta_j = \frac{\delta}{J}$, we thus have the following bound (statement): For all $P_j(\theta)$ ($\forall j=1,\dots,J$), with probability at least $1-\delta$ over the choice of the training set $S$, the folowing holds: 
\begin{align}
\forall Q(\theta), \ \ \mathbb{E}_{Q(\theta)}[l_D(\theta)] \leq \mathbb{E}_{Q(\theta)}[l_S(\theta)] + \sqrt{\frac{\textrm{KL}(Q(\theta)||P_j(\theta)) + \log\frac{n}{\delta} + \log J}{2(n-1)}}, \ \ \ \ \forall j=1,\dots,J.
\label{appeq:pac_bayes_combined}
\end{align}

Now, we choose the prior $P_j(\theta)$ from the prior set that is as close to the posterior  $Q(\theta)$ as possible (in KL divergence).  
Since
\begin{align}
\textrm{KL}(Q||P_j) = \frac{1}{2} \bigg( 
\textrm{Tr}\big(F(\overline{\theta_j})F(\theta_0)^{-1}\big) + \frac{1}{\rho^2}  (\theta_0-\overline{\theta_j})^\top F(\overline{\theta_j})(\theta_0-\overline{\theta_j}) + \log \frac{\vert F(\theta_0)\vert}{\vert F(\overline{\theta_j}) \vert} - k \bigg),
\label{appeq:kl}
\end{align}
if we choose $j^*$ such that $\theta_0 \in R_{j^*}$, using $\log \frac{\vert F(\theta_0)\vert}{\vert F(\overline{\theta_j})\vert} = \sum_i \log \frac{\lambda_i(F(\theta_0))}{\lambda_i(F(\overline{\theta_j}))} \leq \sum_i \log (1+c_{ij}) \leq \sum_i c_{ij} \leq k \epsilon_{min}$ (from Assumption~\ref{a2}) and $\textrm{Tr}(F(\overline{\theta_j})F(\theta_0)^{-1}) = \textrm{Tr}(I+A^j) \leq k + k \epsilon_{min}$ (from Assumption~\ref{a3}), 
we have the following:
\begin{align}
\textrm{KL}(Q||P_{j^*}) \leq \frac{1}{2} \bigg( k + k \epsilon_{min} + \frac{r_{j^*}^2}{\rho^2} + k \epsilon_{min} - k \bigg) 
= \frac{r_{j^*}^2}{2\rho^2} + k \epsilon_{min} 
\leq \frac{r^2}{2\rho^2} + k \epsilon_{min},
\label{appeq:kl_bound}
\end{align}
where $r \triangleq \max_{1\leq j \leq J} r_j$. 
From (\ref{appeq:pac_bayes_combined}) which holds for $\forall j=1,\dots,J$, we take only the inequality corresponding to $j=j^*$. By slightly rephrasing $\mathbb{E}_{Q(\theta)}
[f(\theta)]$ as  $\mathbb{E}_{\epsilon\sim\mathcal{N}(0,\rho^2 F(\theta_0)^{-1})}
[f(\theta_0+\epsilon)]$ and replacing $\theta_0$ with $\theta$, 
we have the following bound that holds with probability at least $1-\delta$:
\begin{align}
\forall \theta \in \Theta, \ \ \mathbb{E}_{\epsilon\sim\mathcal{N}(0,\rho^2 F(\theta)^{-1})}
[l_D(\theta+\epsilon)] \ \leq \  \mathbb{E}_{\epsilon\sim\mathcal{N}(0,\rho^2 F(\theta)^{-1})}[l_S(\theta+\epsilon)] + \sqrt{\frac{\frac{r^2}{2\rho^2} + k \epsilon_{min} + \log\frac{n}{\delta} + \log J}{2(n-1)}}.
\label{appeq:pac_bayes_exp_ver}
\end{align}

The next step is to bound the expectation in RHS of (\ref{appeq:pac_bayes_exp_ver}) by the worst-case loss, similarly as~\cite{sam}. We make use of the following result from~\cite{laurent_massart}:
\begin{align}
z \sim \mathcal{N}(0,\rho^2 I) \ \  \Longrightarrow \ \ \|z\|^2 \leq k \rho^2 \Bigg( 1 + \sqrt{\frac{\log n}{k}} \Bigg)^2 \ \ \textrm{with probability at least} \ 1 - \frac{1}{\sqrt{n}}. 
\label{appeq:laurent_massart}
\end{align}
Since we have non-spherical Gaussian $\epsilon$, we cannot directly apply (\ref{appeq:laurent_massart}), but need some transformation to whiten the correlation among the variables in $\epsilon$. By letting $u=F(\theta)^{1/2} \epsilon$, we have $u \sim \mathcal{N}(0, \rho^2 I)$. Then applying (\ref{appeq:laurent_massart}) leads to $\|u\|^2 = \epsilon^\top F(\theta) \epsilon \leq k \rho^2 ( 1 + \sqrt{(\log n) / k} )^2$ with probability at least $1-1/\sqrt{n}$. We denote the rightmost term by $\gamma^2$, that is, $\gamma = \rho \cdot ( \sqrt{k} + \sqrt{\log n} )$. We then upper-bound $\mathbb{E}_{\epsilon\sim\mathcal{N}(0,\rho^2 F(\theta)^{-1})}[l_S(\theta+\epsilon)]$ by partitioning the $\epsilon$ space into those with $\epsilon^\top F(\theta) \epsilon \leq \gamma^2$ and the rest $\epsilon^\top F(\theta) \epsilon > \gamma^2$. By taking the maximum loss for the former while choosing the loss bound (constant) $l_{max}$ for the latter, we have:
\begin{align}
\mathbb{E}_{\epsilon\sim\mathcal{N}(0,\rho^2 F(\theta)^{-1})}[l_S(\theta+\epsilon)] \ \leq \ (1 - 1/\sqrt{n}) \max_{\epsilon^\top F(\theta) \epsilon \leq \gamma^2} l_S(\theta+\epsilon) + \frac{l_{max}}{\sqrt{n}} \ \leq \ \max_{\epsilon^\top F(\theta) \epsilon \leq \gamma^2} l_S(\theta+\epsilon) + \frac{l_{max}}{\sqrt{n}}.
\label{appeq:exp_loss_bound}
\end{align}
Plugging (\ref{appeq:exp_loss_bound}) and $\gamma$ into (\ref{appeq:pac_bayes_exp_ver}) yields: With probability at least $1-\delta$, $\forall \theta \in \Theta$, the following holds
\begin{align}
\mathbb{E}_{\epsilon\sim\mathcal{N}(0,\rho^2 F(\theta)^{-1})}
[l_D(\theta+\epsilon)] \ &\leq \  \max_{\epsilon^\top F(\theta) \epsilon \leq \gamma^2} l_S(\theta+\epsilon) + \frac{l_{max}}{\sqrt{n}} + \sqrt{\frac{\frac{r^2 ( \sqrt{k} + \sqrt{\log n} )^2}{2\gamma^2} 
+ k \epsilon_{min} + \log\frac{n}{\delta} + \log J}{2(n-1)}} \label{appeq:pac_bayes_max_ver_1} \\
& = \max_{\epsilon^\top F(\theta) \epsilon \leq \gamma^2} l_S(\theta+\epsilon) + \sqrt{\frac{O(k + \log \frac{n}{\delta})}{n-1}}.
\label{appeq:pac_bayes_max_ver_2}
\end{align}
\end{proof}

\section{Approximate Equality of KL Divergence and Fisher Quadratic}\label{appsec:kl_fisher}

In this section we prove that  $d(\theta+\epsilon,\theta) \approx \epsilon^\top F(\theta) \epsilon$ when $\epsilon$ is small, where
\begin{align}
d(\theta',\theta) = \mathbb{E}_x\big[\textrm{KL}(p(y|x,\theta')||p(y|x,\theta))\big] \ \ \ \ \textrm{and} \ \ \ \ F(\theta) = \mathbb{E}_x\mathbb{E}_\theta 
\big[
\nabla \log p(y|x,\theta) \nabla \log p(y|x,\theta)^\top \big]. 
\label{appeq:kld_and_fisher_info}
\end{align}
Note that $\mathbb{E}_\theta[\cdot]$ indicates expectation over $p(y|x,\theta)$. First, we let $d(\theta',\theta; x) = \textrm{KL}(p(y|x,\theta')||p(y|x,\theta))$ and $F(\theta;x) = \mathbb{E}_\theta 
\big[\nabla \log p(y|x,\theta) \nabla \log p(y|x,\theta)^\top \big]$. From the definition of the KL divergence, 
\begin{align}
d(\theta+\epsilon,\theta; x) 
= \sum_y p(y|x,\theta+\epsilon) \cdot \log \frac{p(y|x,\theta+\epsilon)}{p(y|x,\theta)}.
\label{appeq:kl_def}
\end{align}
Regarding $p(y|x,\theta+\epsilon)$ and $\log p(y|x,\theta+\epsilon)$ as functions of $\theta$, we apply the first-order Taylor expansion at $\theta$ to each function as follows:
\begin{align}
p(y|x,\theta+\epsilon) \approx p(y|x,\theta) + \nabla_\theta p(y|x,\theta)^\top \epsilon \ \ \ \ \textrm{and} \ \ \ \ \log p(y|x,\theta+\epsilon) \approx \log p(y|x,\theta) + \nabla_\theta \log p(y|x,\theta)^\top \epsilon.
\label{appeq:taylor}
\end{align}
Plugging these approximates to (\ref{appeq:kl_def}) leads to:
\begin{align}
d(\theta+\epsilon,\theta; x) \ &\approx \ \sum_y \Big( p(y|x,\theta) + \epsilon^\top \nabla_\theta p(y|x,\theta) \Big) \cdot \nabla_\theta \log p(y|x,\theta)^\top \epsilon \label{appeq:kl_deriv_1} \\
&= \ \mathbb{E}_\theta\big[\nabla_\theta \log p(y|x,\theta)\big]^\top \epsilon \ + \ \epsilon^\top \mathbb{E}_\theta\big[\nabla_\theta \log p(y|x,\theta) \nabla_\theta \log p(y|x,\theta)^\top\big] \epsilon, \label{appeq:kl_deriv_2}
\end{align}
where in (\ref{appeq:kl_deriv_2}) we use $\nabla_\theta p(y|x,\theta) = p(y|x,\theta) \nabla_\theta \log p(y|x,\theta)$. 
The first term of (\ref{appeq:kl_deriv_2}) equals $0$ since
\begin{align}
\mathbb{E}_\theta\big[\nabla_\theta \log p(y|x,\theta)\big] = \sum_y p(y|x,\theta) \cdot \frac{\nabla_\theta p(y|x,\theta)}{p(y|x,\theta)} = \sum_y \nabla_\theta p(y|x,\theta) = \frac{\partial}{\partial\theta} \sum_y p(y|x,\theta) = 0.
\end{align}
Lastly, taking expectation over $x$ in (\ref{appeq:kl_deriv_2}) completes the proof.


\end{document}